\documentclass[]{article} 
\usepackage{nips12submit_e,times}

\usepackage{url}
\usepackage{theorem}
\usepackage{amssymb}
\usepackage{amsmath}
\usepackage{inputenc}\inputencoding{latin1}

\newtheorem{theorem}{Theorem}
\newtheorem{lemma}[theorem]{Lemma}

\newtheorem{definition}[theorem]{Definition}
{\theorembodyfont{\rm} \newtheorem{eXample}[theorem]{Example}}

\newcommand{\QED}{\square}
\newcommand{\ENDEX}{$\blacksquare$}
\newenvironment{proof}{\noindent{\em Proof:}}{\hfill$\QED$\linebreak\smallskip}

\title{%
Learning Chordal Markov Networks by 
\\ Constraint Satisfaction} 

\author{Jukka Corander \\ University of Helsinki \\
Finland
\And
Tomi Janhunen\thanks{Also affiliated with the Helsinki Institute of Information Technology, Finland.} \\
Aalto University \\
Finland
\And
Jussi Rintanen\thanks{Also affiliated with Griffith University, Brisbane, Australia, and the Helsinki Institute of Information Technology, Finland. This work was funded by the Academy of Finland (Finnish Centre of Excellence in Computational Inference Research COIN, 251170).} \\
Aalto University \\
Finland
\And
Henrik Nyman\thanks{This work was funded by the Foundation of \AA bo Akademi University, as part of the grant for the Center of Excellence in Optimization and Systems Engineering.} \\
\AA bo Akademi University \\
Finland
\And
Johan Pensar$^\ddag$ \\ 
\AA bo Akademi University \\
Finland
}

\nipsfinalcopy

\begin{document} 
 
\maketitle 
 
\begin{abstract} 
We investigate the problem of learning the structure of a Markov
network from data.
It is shown that the structure of such networks can be described in
terms of constraints which enables the use of existing solver
technology with optimization capabilities to compute optimal networks
starting from initial scores computed from the data.
To achieve efficient encodings, we develop a novel characterization of
Markov network structure using a balancing condition on the separators
between cliques forming the network. The resulting translations into
propositional satisfiability and its extensions such as maximum
satisfiability, satisfiability modulo theories, and answer set
programming, enable us to prove optimal certain network structures
which have been previously found by stochastic search.\footnote{This paper has been accepted for publication in the proceedings of the Neural Information Processing Systems conference NIPS'2013.}
\end{abstract}


\newcommand{\impl}{\rightarrow}
\newcommand{\disj}{\vee}
\newcommand{\conj}{\wedge}
\newcommand{\eqvi}{\leftrightarrow}
\newcommand{\setd}{\backslash}
\newcommand{\edges}[1]{\mbox{edges}(#1)}

\newcommand{\sys}[1]{\textsc{#1}}
\newcommand{\ds}[1]{\textbf{#1}}


\newcommand{\TODO}[1]{%
\vspace{1\baselineskip}\noindent\hrulefill\hspace{1em}{#1}%
\hspace{1em}\hrulefill\vspace{1\baselineskip}}
\newcommand{\DONE}{%
\vspace{1\baselineskip}\noindent\hrule\vspace{1\baselineskip}}
\newcommand{\el}[1]{\vspace{#1\baselineskip}}



\section{Introduction}
\label{section:introduction}

Graphical models (GMs) represent the backbone of the generic
statistical toolbox for encoding dependence structures in multivariate
distributions. Using Markov networks or Bayesian networks conditional
independencies between variables can be readily communicated and used
for various computational purposes. The development of the statistical
theory of GMs is largely set by the seminal works of
Darroch et al.~\cite{darroch80} and
Lauritzen and Wermuth \cite{lauritzen89}.
Although various approaches have been developed to generalize the
theory of graphical models to allow for modeling of more complex
dependence structures, Markov networks and Bayesian networks are still
widely used in applications ranging from genetic mapping of diseases
to machine learning and expert systems.

Bayesian learning of undirected GMs, also known as \emph{Markov random
  fields}, from databases has attained a considerable interest, both
in the statistical and computer science literature
\cite{corander03,corander08,dellaportas99,giudici03,giudici99,koivisto04,madigan94}.
The cardinality and complex topology of GM space pose difficulties
with respect to both the computational complexity of the learning
task and the reliability of reaching representative model structures.
Solutions to these problems have been proposed in earlier work.
Della Pietra et al. \cite{DellaPietraDPL97} present a greedy local search
algorithm Markov network learning and apply it to discovering word morphology.
Lee et al. \cite{LeeGK06} reduce the learning problem to a convex
optimization problem that is solved by gradient descent.
Related methods have been investigated later \cite{SchmidtNMM07,HoeflingTibshirani09}.


Certain types of stochastic search methods, such as Markov Chain Monte
Carlo (MCMC) or simulated annealing can be proven to be consistent with
respect to the identification of a structure maximizing posterior
probability \cite{corander08,dellaportas99,giudici03,giudici99}.
However, convergence of such methods towards the areas
associated with high posterior probabilities may still be slow when
the number of nodes increases \cite{corander08,giudici03}.
In addition, it is challenging to
guarantee that the identified model indeed truly represents the global
optimum since the consistency of MCMC estimates is by definition a limit result.
To the best of our knowledge, strict constraint-based search
methods have not been previously applied in learning of
Markov random fields. 
In this article, we formalize the structure of Markov networks using
constraints at a fairly general level. This enables the development of
reductions from the structure learning problem to
\emph{propositional satisfiability} (SAT) \cite{HandbookOfSAT2009}
and its generalizations such as
\emph{maximum satisfiability} (MAXSAT) \cite{LM09HBSAT},
and \emph{satisfiability modulo theories} (SMT) \cite{BSST09HBSAT},
as well as \emph{answer-set programming} (ASP) \cite{BET11:cacm};
and the deployment of respective solver technology for computations.
A main novelty is the recognition of maximum weight spanning trees of
the clique graph by a condition on the cardinalities of occurrences of
variables in cliques and separators, which we call the {\em balancing
condition}.

The article is structured as follows.
We first review some details of Markov networks and the respective
structure learning problem in Section \ref{section:structure-learning}.
To enable efficient encodings of Markov network learning as a constraint
satisfaction problem,
in Section \ref{section:fundamentals}
we establish a new characterization of the separators of a Markov network based on a \emph{balancing condition}.
In Section \ref{section:constraints}, we provide a high-level
description how the learning problem can be expressed using
constraints and sketch the actual translations into propositional
satisfiability (SAT) and its generalizations.
We have implemented these translations and conducted 
experiments to study the performance of existing solver technology
on structure learning problems in Section \ref{section:experiments}
using two widely used datasets \cite{Whittaker90}.
Finally, some conclusions and possibilities for further research in
this area are presented in Section \ref{section:conclusions}.


\section{Structure Learning for Markov Networks}
\label{section:structure-learning}

An undirected graph $G = (N,E)$ consists of a set of \textit{nodes}
$N$ which represents a set of random variables and a set of
\textit{undirected edges} $E\subseteq\{N \times N\}$. A \textit{path}
in a graph is a sequence of nodes such that every two consecutive nodes are
connected by an edge. Two sets of nodes $A$ and $B$ are said to be
\textit{separated} by a third set of nodes $D$ if every path between a
node in $A$ and a node in $B$ contains at least one node in $D$. An
undirected graph is \textit{chordal} if for all paths $v_0,
\ldots v_n$ with $n \geq 4$ and $v_0 = v_n$ there exists two nodes
$v_i$, $v_j$ in the path connected by an edge such that $j \neq i \pm
1$. A \textit{clique} in a graph is a set of nodes $c$ such that every
two nodes in it are connected by an edge. In addition, there
may not exist a set of nodes $c'$ such that $c \subset c'$ and every two
nodes in $c'$ are connected by an edge. Given the set of cliques $C$
in a chordal graph, the set of \textit{separators} $S$ can be obtained
through intersections of the cliques ordered in terms of a junction
tree \cite{Golumbic80}, this operation is considered thoroughly in
Section \ref{section:fundamentals}.

A Markov network is defined as a pair consisting of a graph $G$ and a
joint distribution $P_{N}$ over the variables in $N$. The graph
specifies the dependence structure of the variables and $P_{N}$
factorizes according to $G$ (see below). Given $G$ it is possible to
ascertain if two sets of variables $A$ and $B$ are {\em conditionally
independent} given another set of variables $D$, due to the global
Markov property
\[
A \perp \hspace{-0.17cm} \perp B \mid D,
\text{ if } D \text{ separates } A \text{ from }B.
\]
For a Markov network with a chordal graph $G$, the probability of a
joint outcome $x$ factorizes as
\[
P_N(x)=\frac{\prod_{c_i \in C } P_{c_i}(x_{c_i})}{\prod_{s_i \in S} P_{s_i} (x_{s_i}) }.
\]
Following this factorization the marginal likelihood of a dataset
$\textbf{X}$ given a Markov network with a chordal graph $G$ can be
written
\[
P(\textbf{X}|G) =
\frac{\prod_{c_i \in C } P_{c_i}(\textbf{X}_{c_i})}
     {\prod_{s_i \in S} P_{s_i} (\textbf{X}_{s_i}) }.
\]
By a suitable choice of prior distribution, the terms
$P_{c_i}(\textbf{X}_{c_i})$ and $P_{s_i} (\textbf{X}_{s_i})$ can be
calculated analytically. Let $a$ denote an arbitrary clique or
separator containing the variables $X_a$ whose outcome space has the
cardinality $k$. Further, let $n_{a}^{(j)}$ denote the number of
occurrences where $X_a = x_{a}^{(j)}$ in the dataset
$\mathbf{X}_{a}$. Now assign the Dirichlet $(\alpha_{a_{1}}
,\ldots,\alpha_{a_{k}})$ distribution as prior over the probabilities
$P_a(X_a = x_{a}^{(j)}) = \theta_j$, determining the distribution
$P_a(X_a)$. Given these settings $P_{a} ( \mathbf{X}_{a} )$ can be
calculated as
\[
P_{a}(\mathbf{X}_{a})=
\int_{\Theta}\prod_{j=1}^{k}(\theta_{j})^{n_{a}^{(j)}}\cdot\pi_{a}(\theta)d\theta
\]
where $\pi_{a}(\theta)$ is the density function of the Dirichlet prior
distribution. By the standard properties of the Dirichlet integral,
$P_{a} ( \mathbf{X}_{a} )$ can be reduced to the form
\[
P_{a}(\mathbf{X}_{a})=
\frac{\Gamma(\alpha)}{\Gamma(n_a+\alpha)}
\prod_{j=1}^{k}\frac{\Gamma(n_{a}^{(j)}+\alpha_{a_{j}})}{\Gamma(\alpha_{a_{j}})}
\]
where $\Gamma(\cdot)$ denotes the gamma function and
\[
\alpha = \sum_{j=1}^{k}\alpha_{a_{j}} \qquad \text{and}
\qquad n_a = \sum_{j=1}^{k}n_{a}^{(j)}.
\]
When dealing with the marginal likelihood of a dataset it is most
often necessary to use the logarithmic value $\log
P(\textbf{X}|G)$. Introducing the notations $v(c_i) = \log
P_{c_i}(\textbf{X}_{c_i})$ the logarithmic value of the marginal
likelihood can be written
\begin{equation}
\label{eq:log-likelihood}
\log P(\textbf{X}|G) =
\sum_{c_i \in C } \log P_{c_i}(\textbf{X}_{c_i}) -
\sum_{s_i \in S} \log P_{s_i} (\textbf{X}_{s_i}) =
\sum_{c_i \in C } v(c_i) - \sum_{s_i \in S} v(s_i).
\end{equation}
The learning problem is to find a graph structure $G$ that
optimizes the posterior distribution
\[
P(G|\textbf{X}) =
\frac{P(\textbf{X}|G) P(G)}{\sum_{G \in \mathcal{G}} P(\textbf{X}|G) P(G)}.
\] 
Here $\mathcal{G}$ denotes the set of all graph structures under
consideration and $P(G)$ is the prior probability assigned to $G$. In
the case where a uniform prior is used for the graph structures the
optimization problem reduces to finding the graph with the largest
marginal likelihood.


\section{Fundamental Properties and Characterization Results}
\label{section:fundamentals}

In this section, we point out some properties of chordal graphs and
clique graphs that can be utilized in the encodings of the learning
problem. In particular, we develop a characterization of maximum
weight spanning trees in terms of a \emph{balancing condition} on
separators.

The separators needed for determining the score
(\ref{eq:log-likelihood}) of a candidate Markov network are defined as
follows.  Given the cliques, we can form the \emph{clique graph}, in
which the nodes are the cliques and there is an edge between two nodes
if the corresponding cliques have a non-empty intersection.  We label
each of the edges with this intersection and consider the cardinality
of the label as its \emph{weight}.
The \emph{separators} are the edge labels of a
\emph{maximum weight spanning tree}
of the clique graph. Maximum weight spanning trees of arbitrary graphs
can be found in polynomial time by reducing the problem to finding
\emph{minimum weight spanning trees}.  This reduction consists of
negating all the edge weights and then using any of the polynomial
time algorithms for the latter problem \cite{GrahamHell85}.
There may be several maximum weight spanning trees, but they induce
exactly the same separators, and they only differ in terms of which
pairs of cliques induce the separators.

To restrict the search space we can observe that
a chordal graph with $n$ nodes has at
most $n$ maximal cliques \cite{Golumbic80}.
This gives an immediate upper bound on
the number of cliques chosen to build a Markov network, which can
be encoded as a simple cardinality constraint.
%
%

\subsection{Characterization of Maximum Weight Spanning Trees}
\label{section:balancing}

To simplify the encoding of maximum weight spanning trees (and
forests) of chordal clique graphs, we introduce the notion of
\emph{balanced spanning trees} (respectively, forests), and show that
these two concepts coincide when the underlying graph is chordal.
Then separators can be identified more effectively: rather than
encoding relatively complex algorithms for finding maximum-weight
spanning trees as constraints, it is sufficient to select a subset of
the edges of the clique graph that is acyclic and satisfies the
balancing condition expressible as a cardinality constraint over
occurrences of nodes in cliques and separators.

\begin{definition}[Balancing Condition]\label{def:balancing-condition}
A spanning tree (or forest) of a clique graph is \emph{balanced} if for
every node $n$, the number of cliques containing $n$ is one higher
than the number of labeled edges containing $n$.
\end{definition}

While in the following we state many results for spanning trees only,
they can be straightforwardly generalized to spanning forests as well
(in case the Markov networks are disconnected.)


\begin{lemma}\label{le:eqbalancing}
For any clique graph, all its balanced spanning trees have the same weight.
\end{lemma}

\begin{proof}
This holds in general because the balancing condition requires
exactly the same number of occurrences of any node in the separator
edges for any balanced spanning tree, and the weight is defined as
the sum of the occurrences of nodes in the edge labels.
\end{proof}

\begin{lemma}[\cite{Shibata88,JensenJensen94}]\label{le:maximumspanningjunction}
Any maximum weight spanning tree of the clique graph is a junction tree,
and hence satisfies the {\em running intersection property}:
for every pair of nodes $c$ and $c'$,
$(c\cap c')\subseteq c''$ for all nodes $c''$ on the unique path between
$c$ and $c'$.
\end{lemma}

\begin{lemma}\label{le:max2bala}
Let $T = \langle V,E_T\rangle$ be a maximum weight spanning tree of
the clique graph $\langle V,E\rangle$ of a connected chordal graph.
Then $T$ is balanced.
\end{lemma}

\begin{proof}
We order the tree by choosing an arbitrary clique as the root and
by assigning a depth to all nodes according to their distance from the
root node.  The rest of the proof proceeds by induction on the
height of subtrees starting from the leaf nodes as the base case.
The induction hypothesis says that all subtrees satisfy the balancing condition.
The base cases are trivial: each leaf node (clique) trivially
satisfies the balancing condition, as there are no separators to
consider.

In the inductive cases, we have a clique $c$ at depth $d$, connected
to one or more subtrees rooted at neighboring cliques  $c_1,\ldots,c_k$
at depth $d+1$, with the subtrees satisfying the balancing condition.
We show that the tree consisting of the clique $c$, the labeled edges
connecting $c$ respectively to cliques $c_1,\ldots,c_k$, and the
subtrees rooted at $c_1,\ldots,c_k$, satisfies the balancing condition.

First note that by Lemma \ref{le:maximumspanningjunction}, any maximum
weight spanning tree of the clique graph is a junction tree and hence
satisfies the running intersection property, meaning that for any two
cliques $c_1$ and $c_2$ in the tree, every clique
on the unique path connecting them includes
$c_1\cap c_2$.

We have to show that the subtree rooted at $c$ is balanced, given that
its subtrees are balanced. We show that the balancing condition is
satisfied for each node separately. So let $n$ be one of the nodes in
the original graph.  Now each of the subtrees rooted at some $c_i$ has
either 0 occurrences of $n$, or $k_i\leq 1$ occurrences in the cliques
and $k_i-1$ occurrences in the edge labels, because by the induction
hypothesis the balancing condition is satisfied.
In total, four cases arise:
\begin{enumerate}
\item
The node $n$ does not occur in any of the subtrees.

Now the balancing condition is trivially satisfied for the subtree rooted
at $c$, because $n$ either does not occur in $c$, or it occurs in $c$
but does not occur in the label of any of the edges to the subtrees.

\item
The node $n$ occurs in more than one subtree.

Since any maximum weight spanning tree is a junction tree by Lemma
\ref{le:maximumspanningjunction}, $n$ must occur also in $c$
and in the labels of the edges between $c$ and the cliques in which the
subtrees with $n$ are rooted.
Let $s_1,\ldots,s_j$ be the numbers of occurrences of $n$ in the edge
labels in the subtrees with at least one occurrence of $n$, and
$t_1,\ldots,t_j$ the numbers of occurrences of $n$ in the cliques in
the same subtrees.

By the induction hypothesis, these subtrees are balanced, and hence
$t_i-s_i=1$ for all $i\in\{1,\ldots,j\}$.
The subtree rooted at $c$ now has $1+\sum_{i=1}^kt_i$ occurrences of
$n$ in the nodes (once in $c$ itself and then the subtrees) and
$j+\sum_{i=1}^js_i$ occurrences in the edge labels, where the $j$
occurrences are in the edges between $c$ and the $j$ subtrees.

We establish the balancing condition through a sequence of equalities.
The first and the last expression are the two sides of the condition.
\[
\begin{array}{cll}
\multicolumn{2}{l}{(1+\sum_{i=1}^jt_i)-(j+\sum_{i=1}^ks_i)} \\
& = 1-j+\sum_{i=1}^j(t_i-s_i) & \mbox{reordering the terms} \\
& = 1-j+j & \mbox{since } t_i-s_i=1 \mbox{ for every subtree} \\
& = 1 \\
\end{array}
\]
Hence also the subtree rooted at $c$ is balanced.

\item
The node $n$ occurs in one subtree and in $c$.

Let $i$ be the index of the subtree in which $n$ occurs.  Since any
maximum weight spanning tree is a junction tree by Lemma
\ref{le:maximumspanningjunction}, $n$ must occur also in the clique
$c_i$.  Hence $n$ occurs in the label of the edge from $c_i$ to $c$.
Since the subtree is balanced, the new graph obtained by adding the
clique $c$ and the edge with a label containing $n$ is also
balanced. Further, adding all the other subtrees that do not contain
$n$ will not affect the balancing of $n$.

\item
The node $n$ occurs in one subtree but not in $c$.

Since there are $n$ occurrences of $n$ in any of the other subtrees,
in $c$, or in the edge labels between $c$ and any of the subtrees,
the balancing condition holds.
\end{enumerate}
This completes the induction step and consequently,
the whole spanning tree is balanced.
\end{proof}

\begin{lemma}\label{le:bala2max}
Assume $T = \langle V,E_B\rangle$ is a spanning tree of
the clique graph $G_C=\langle V,E\rangle$ of a chordal graph
that satisfies the balancing condition.
Then $T$ is a maximum weight spanning tree of $G_C$.
\end{lemma}

\begin{proof}
Let $T_M$ be one of the spanning trees of $G_C$ with the maximum weight $w$.
By Lemma \ref{le:max2bala}, this maximum weight spanning tree is balanced.
By Lemma \ref{le:eqbalancing}, $T$ has the same weight $w$
as $T_M$. Hence also $T$ is a maximum weight spanning tree of $G_C$.
\end{proof}

Lemmas \ref{le:max2bala} and  \ref{le:bala2max} directly yield the
following.

\begin{theorem}
For any clique graph of a chordal graph, any of its subgraphs is a maximum
weight spanning tree if and only if it is a balanced acyclic subgraph.
\end{theorem}



\section{Constraints and Their Translations into
MAXSAT, SMT, and ASP}
\label{section:constraints}

The objectives of this section are twofold. First, we show how the
structure learning problem of Markov networks is cast in
an abstract constraint satisfaction problem.
Secondly, we partly formalize the constraints involved in the language
of propositional logic. This is the language directly supported by
SMT solvers and straightforward to transform into conjunctive normal form
used by SAT and MAXSAT solvers.  In ASP, however, slightly different
rule-based formulations are used but we omit corresponding ASP rules
for space reasons.

The learning problem is formalized as follows. The goal is to find a
\emph{balanced} spanning tree
(cf.~Definition \ref{def:balancing-condition})
for a set $C$ of cliques forming a Markov network and
the set $S$ of separators induced by the tree structure.
In addition, $C$ and $S$ are supposed to be optimal
in the sense of (\ref{eq:log-likelihood}), i.e.,
the overall \emph{score}
$v(C,S)=\sum_{c\in C}v(c)-\sum_{s\in S}v(s)$
is maximized. The individual score $v(c)$ for any set of nodes $c$
describes how well the mutual dependence of the variables in $c$
reflected by the data.

\begin{definition}\label{def:slb-as-csp}
Let $N$ be a set of nodes representing random variables and
$v:{\mathbf{2}}^N\rightarrow\mathbb{R}$
a scoring function. A \emph{solution} to the Markov network
learning problem is a set of \emph{cliques} $C=\{c_1,\ldots,c_n\}$ 
satisfying the following requirements viewed as abstract constraints:

\begin{enumerate}
\item\label{item:coverage}
Every node is included in at least one of the chosen cliques in $C$, i.e.,
$\bigcup_{i=1}^nc_i = N$.

\item\label{item:maximal}
Cliques in $C$ are maximal, i.e.,
\begin{enumerate}
\item for every $c,c'\in C$, if $c\subseteq c'$, then $c=c'$; and
\item for every $c\subseteq N$,
      if $\edges{c}\subseteq\bigcup_{c'\in C}\edges{c'}$,
      then $c\subseteq c'$ for some $c'\in C$
\end{enumerate}
where $\edges{c}=\{\{n,n'\}\subseteq c \mid n\neq n'\}$
is defined for each $c\subseteq N$.

\item\label{item:chordal}
The graph $\langle N,E\rangle$ with the set of edges
$E=\bigcup_{c\in C}\edges{c}$ is chordal.

\item\label{item:balanced-spanning-tree}
The set $C$ has a balanced spanning tree
labeled by a set of \emph{separators} $S=\{s_1,\ldots,s_m\}$.
\end{enumerate}
Moreover, the solution is \emph{optimal} if it maximizes
the overall score $v(C,S)$.
\end{definition}

In what follows, the encodings of basic graph properties
(Items \ref{item:coverage} and \ref{item:maximal} above) are
worked out in Section \ref{section:graph-properties}.
The more complex properties (Items \ref{item:chordal} and
\ref{item:balanced-spanning-tree}) are addressed in Sections
\ref{section:chordality} and \ref{section:separators}.

\subsection{Graph Properties}
\label{section:graph-properties}

We assume that clique candidates which are the non-empty subsets of
$V$ are indexed from 1 to $2^{|V|}$ and, from time to time, we
identify a clique with its index. Moreover, each clique candidate
$c\subseteq V$ has a score $v(c)$ associated with it.
To encode the search space for Markov networks, we introduce,
for every clique candidate $c$, a propositional variable $x_c$ 
denoting that $c$ is part of the learned network.
For every node $n$, we have the constraint
\begin{equation}
\label{eq:node-coverage}
x_{c_1}\disj\cdots\disj x_{c_m}
\end{equation}
where $c_1,\ldots,c_m$ are all cliques $c$ with $n\in c$. This clause
formalizes
Item \ref{item:coverage} of Definition~\ref{def:slb-as-csp}.
For Item \ref{item:maximal} (a) and each pair of clique candidates $c$
and $c'$ such that $c\subset c'$, we need
\begin{equation}
\label{eq:anti-chain}
\neg x_{c}\disj\neg x_{c'}
\end{equation}
for the mutual exclusion of $c$ and $c'$ in the network.
The second part (b) of Item \ref{item:maximal} means that
any implicit clique structure $c$ created by the edges of chosen cliques
must be covered
by some proper superset, i.e., a chosen clique $c'$, and otherwise
the clique $c$ must be chosen itself ($c'=c$).
%
To formalize this constraint, we introduce further propositional
variables $e_{n,m}$ that represent edges $\{n,m\}$ that are
in at least one chosen clique, and we require that if a clique's edges
are chosen, then none of its proper subset cliques can be chosen.
\footnote{%
As the edges are undirected, we limit to $e_{n,m}$ such that the
ordering of $n$ and $m$ according to some fixed ordering is
increasing, i.e., $n<m$. Under this assumption, $e_{m,n}$ for $n<m$
denotes $e_{n,m}$.}
Hence, for any $\{n,m\}\subseteq N$ such that $n<m$, let
$c_1,\ldots,c_k$ be all cliques such that $\{n,m\}\subseteq c_i$.
Then we introduce the constraint
\begin{equation}
\label{eq:edges}
e_{n,m}\eqvi(x_{c_1}\disj\cdots\disj x_{c_k})
\end{equation}
to make the edges of the chosen cliques explicit.  Furthermore, for
every clique candidate $c = \{n_1,\ldots,n_k\}$ and every node
$n\in V\setd c$ we need a constraint
\begin{equation}
\label{eq:implicit}
x_{c}\impl(\neg e_{n_1,n}\disj\cdots\disj\neg e_{n_k,n})
\end{equation}
where $e_{n_1,n},\ldots,e_{n_k,n}$ represent all additional edges that
would turn $c\cup\{n\}$ into a clique.
These constraints guarantee that the maximality of chosen cliques in
the sense of Item \ref{item:maximal} in Definition \ref{def:slb-as-csp}.

\subsection{Chordality}
\label{section:chordality}

We use a straightforward encoding of the chordality condition. The
idea is to generate constraints corresponding to every $k\geq 4$
element subset $S=\{n_1,\ldots,n_k\}$ of $N$.
Let us consider all cycles these nodes could form in the graph
$\langle N,E\rangle$ of Item \ref{item:chordal} in Definition
\ref{def:slb-as-csp}.  A cycle starts from a given node, goes through
all other nodes in some order, with (undirected) edges between two
consecutive nodes, and ends in the starting node. The number
of constraints required can be reduced by two observations.
First, the same cycle could be generated from different starting
nodes, e.g., cycles $n_1,n_2,n_3,n_4,n_1$ and $n_2,n_3,n_4,n_1,n_2$
are the same.
Second, generating the same cycle in two opposite directions, as in
$n_1,n_2,n_3,n_4,n_1$ and $n_1,n_4,n_3,n_2,n_1$, is clearly unnecessary.
To avoid redundant cycle constraints, we arbitrarily fix the starting
node, and additionally require that the index of the second node in
the cycle is lower than the index of the second last node. These
restrictions guarantee that every cycle associated with $S$ is
considered exactly once.
Now, the chordality constraint says that if there is an edge between
every pair of consecutive nodes in $n_1,\ldots,n_k,n_1$, then there
also has to be an edge between at least one pair of two
non-consecutive nodes. In the case $k=4$, for instance, this leads to
formulas of the form
\begin{equation}
\label{eq:chordal-formula}
e_{n_1,n_2}\conj e_{n_2,n_3}\conj e_{n_3,n_4}\conj e_{n_4,n_1}
\impl e_{n_1,n_3}\disj e_{n_2,n_4}.
\end{equation}
It is clear that the encoding of chordality constraints is exponential
in $|N|$ and therefore not scalable to very large numbers of
nodes. However, the datasets considered in Section
\ref{section:experiments} have only 6 or 8 variables, and in these cases
the exponentiality is not a problem. It is also possible to further
condense the encoding using \emph{cardinality constraints} available
in some constraint languages such as ASP.

\subsection{Separators}
\label{section:separators}

Separators for pairs $c$ and $c'$ of clique candidates can be
formalized as propositional variables $s_{c,c'}$, meaning
that $c\cap c'$ is a separator and there is an edge in
the spanning tree between $c$ and $c'$ labeled by $c\cap c'$.
The corresponding constraint is
\begin{equation}
\label{eq:spanning}
s_{c,c'}\impl x_c\conj x_{c'}.
\end{equation}
The lack of the converse implication formalizes the \emph{choice} of
the spanning tree, i.e., $s_{c,c'}$ can be false even if $x_c$ and
$x_{c'}$ are true. The remaining constraints on separators
fall into two cases.

First, we have cardinality constraints encoding the balancing
condition (cf. Section \ref{section:balancing}): each variable occurs
in the chosen cliques one more time than it occurs in the separators
labeling the spanning tree. As stated above, cardinality constraints
are natively supported by some constraint languages or, alternatively,
they can be efficiently reduced to disjunctive Boolean constraints
\cite{Sinz05}. For space reasons, we do not present the respective
propositional formulas here.
Second, the separators are not allowed to form a cycle. This property
is not guaranteed by the balancing condition alone so that a separate
encoding of the acyclicity of the graph formed by the cliques with
separators as the edges is needed. Our encoding of the acyclicity condition is
based on an inductive definition of tree structure: we repeatedly
remove \emph{leaf nodes}, i.e., nodes with at most one neighbor, until
all nodes have been removed. When applying this definition
to a graph with a cycle, some nodes will remain in the end.
To encode this, we define the {\em leaf level} for each node in a graph.
A node is a {\em level 0 leaf} iff
it has 0 or 1 neighbors in the graph.
A node is a {\em level $n+1$ leaf} iff
all its neighbors except possibly one are level $j\leq n$ leaves.
This definition is directly expressible by Boolean constraints.
Then, a graph with $m$ nodes is acyclic iff all its nodes
are level $\lfloor \frac{m}{2} \rfloor$ leaves. We use this acyclicity
test for the chosen cliques and separators.


\section{Experimental Evaluation}
\label{section:experiments}

The constraints described in Section \ref{section:constraints} can be
alternatively expressed as MAXSAT, SMT, or ASP problems. In what
follows, we exploit such encodings and respective back-end solvers in
order to compute globally optimal Markov networks for datasets
from the literature.
The test runs were with an Intel Xeon( E3-1230 CPU running at 3.20 GHz.
For both datasets, we computed the respective \emph{score file} that
specifies the score of each clique candidate, i.e., the log-value of
its potential function, and the list of variables involved in that
clique. The score files were then translated into respective encodings
and run on a variety of solvers.
\begin{enumerate}
\item
For the MAXSAT encodings, we tried out \sys{SAT4J} (version 2.3.2)
\cite{LeBerreParrain10} and \sys{PWBO} (version 2.2)
\cite{MartinsML12}. The latter was run in its default configuration
as well as in the UB configuration.

\item
For SMT, we used the \sys{OptiMathSAT} solver
(version 5) \cite{SebastianiTomasi12}.

\item
For ASP, we used the \sys{Clasp} (version 2.1.3)
\cite{GKS12:aij} and \sys{HClasp}%
\footnote{\url{http://www.cs.uni-potsdam.de/hclasp/}}
(also v.~2.1.3) solvers.
The latter allows declaratively specifying search heuristics.
We also tried the \sys{LP2NORMAL} tool that
reduces cardinality constraints to more basic constraints
\cite{JN11:mg65}.
\end{enumerate}
The MAXSAT and ASP solvers only support integer scores obtained
by multiplying the original scores by 1000 and rounding.
The SMT solver OptiMathSAT used the original floating point scores.

\begin{table}[t]
\begin{center}
\begin{tabular}{|l|rr||rr|}
\hline
                 & \ds{heart} &       \ds{econ} & \ds{heart} & \ds{econ} \\
\hline
\sys{OptiMathSAT}     &  $74$ &               - &    3930 kB &    139 MB \\
\sys{PWBO} (default)  & $158$ &               - &    3120 kB &    130 MB \\
\sys{PWBO} (UB)       &  $63$ &               - &    3120 kB &    130 MB \\
\sys{SAT4J}           &  $28$ &               - &    3120 kB &    130 MB \\
\sys{LP2NORMAL+Clasp} & $111$ &               - &    8120 kB &   1060 MB \\
\sys{Clasp}           & $5.6$ &               - &     197 kB &    4.2 MB \\
\sys{HClasp}          & $1.6$ & $310\times10^3$ &     203 kB &    4.2 MB \\
\hline
\end{tabular}
\end{center}
\caption{\label{table:results}%
Summary of results: Runtimes in seconds and
sizes of solver input files}
\end{table}

To illustrate the potential residing in solver technology, we consider
two datasets, one containing risk factors in heart diseases and the
other variables related to economical behavior \cite{Whittaker90},
to be abbreviated by
\ds{heart} and \ds{econ} in the sequel. For \ds{heart}, the globally
optimal network has been verified via (expensive) exhaustive enumeration.
For \ds{econ}, however, exhaustive enumeration is impractical due to the
extremely large search space, and consequently the optimality of
the Markov network found by stochastic search in \cite{corander08} had been
open until now.
The results have been collected in Table \ref{table:results}.
%

The \ds{heart} data involves $6$ variables giving rise to $2^6=64$
clique candidates in total and a search space of
$2^{15}$ undirected networks of which a subset are decomposable. 
For instance, the ASP solver \sys{HClasp} used in our experiments
traversed a considerably smaller search space that consisted of $26651$
(partial) networks.  This illustrates the power of
branch-and-bound type algorithms behind the solvers under
consideration and their ability to cut down the search space
effectively. On the other hand, the \ds{econ} dataset is based on $8$
variables giving rise to a much larger search space $2^{28}$.
We were able to solve this instance optimally with only one solver,
\sys{HClasp}, which allows for a more refined control of the
search heuristic. To this end, we used a quite simple scheme where
cliques are tried in a size ordering, the greatest cliques first. With
this arrangement the global optimum is found in roughly $14$ hours
after which $3$ days is spent on the proof of optimality.


\section{Conclusions}
\label{section:conclusions}
 
Boolean constraint methods appear not to have been earlier applied to
learning of undirected Markov networks.  In this article we introduced
a generic approach in which the learning problem is expressed in terms
of constraints on variables that determine the structure of the
learned network.  The related problem of structure learning of
Bayesian networks has been addressed by general-purpose combinatorial
search methods, including MAXSAT \cite{Cussens08} and a
constraint-programming solver with a linear-programming solver as a
subprocedure \cite{Cussens11,BarlettCussens13}.  We introduced
explicit translations of the generic constrains to the languages of
MAXSAT, SMT and ASP, and demonstrated their use through existing
solver technology. Our method thus opens up a novel venue of research
to further develop and optimize the use such technology for network
learning. A wide variety of possibilities does exist also for using
these methods in combination with stochastic or heuristic search.

\bibliographystyle{unsrt}
\bibliography{networklearning,library,other}

\end{document}